\pdfoutput=1
\documentclass[11pt]{article}

\usepackage{acl}

\usepackage{times}
\usepackage{latexsym}

\usepackage[T1]{fontenc}
\usepackage[utf8]{inputenc}

\usepackage{microtype}

\usepackage[T1]{fontenc}
\usepackage{graphicx}
\usepackage{graphicx}
\usepackage{soul}
\usepackage{multirow}
\usepackage{bbold}
\usepackage{makecell}
\usepackage{amsmath}
\usepackage{hyperref}       % hyperlinks
\usepackage{cleveref}
\usepackage{bm}
\usepackage{xspace}
\usepackage{adjustbox}
\usepackage{subcaption}
\usepackage[utf8]{inputenc}
\usepackage{enumitem}
\usepackage{booktabs}
\usepackage{multirow}
\usepackage{pifont}
\usepackage{algorithm}
\usepackage{algpseudocode}
\usepackage{mathrsfs}
\usepackage{bbm}
\usepackage{enumitem}
\newcommand{\cmark}{\ding{51}}%

\usepackage{amsthm}

\theoremstyle{definition}
\newtheorem{definition}{Definition}[section]
\newcommand\Range[1]{\mathrm{Range}(#1)}

\title{Can Public Large Language Models Help Private Cross-device \\ Federated Learning?}

\author{Boxin Wang$^3$\thanks{ \, Part of the work was done while Boxin Wang was an intern at Google. Correspondence to: Boxin Wang \href{mailto:boxinw2@illinois.edu }{\texttt{boxinw2@illinois.edu }} and Zheng Xu \href{mailto:xuzheng@google.com}{\texttt{xuzheng@google.com}}.}\,,\, 
Yibo Jacky Zhang$^4$, Yuan Cao$^2$, Bo Li$^3$, H. Brendan McMahan$^1$, \\
{\bf Sewoong Oh$^1$, Zheng Xu$^1$, Manzil Zaheer$^2$} \\
$^1$ Google Research, $^2$Google Deepmind, $^3$UIUC, $^4$Stanford}

\usepackage{amsmath,amsfonts,bm, amsthm, mathtools}

\newtheorem{theorem}{Theorem}

\numberwithin{theorem}{section}
\numberwithin{lemma}{section}
\numberwithin{proposition}{section}
\numberwithin{conjecture}{section}

\def\eqref#1{equation~\ref{#1}}
\def\1{\bm{1}}

\DeclareMathAlphabet{\mathsfit}{\encodingdefault}{\sfdefault}{m}{sl}
\SetMathAlphabet{\mathsfit}{bold}{\encodingdefault}{\sfdefault}{bx}{n}

\def\gH{{\mathcal{H}}}

\def\gX{{\mathcal{X}}}

\def\sR{{\mathbb{R}}}

\newcommand{\E}{\mathbb{E}}

\begin{document}
\maketitle
\begin{abstract}
We study (differentially) private federated learning (FL) of language models. The language models in cross-device FL are relatively small, which can be trained with meaningful formal user-level differential privacy (DP) guarantees when massive parallelism in training is enabled by the participation of a moderate size of users. Recently, public data has been used to improve privacy-utility trade-offs for both large and small language models. 
In this work, we provide a systematic study of using large-scale public data and LLMs to help differentially private training of on-device FL models, and further improve the privacy-utility tradeoff by techniques of distillation.
Moreover, we propose a novel distribution matching algorithm with theoretical grounding to sample public data close to private data distribution, which significantly improves the sample efficiency of (pre-)training on public data.
The proposed method is efficient and effective for training private models by taking advantage of public data, especially for customized on-device architectures that do not have ready-to-use pre-trained models. 
\end{abstract}

\section{Introduction}

Federated Learning (FL) \citep{fl,mcmahan18learning,kairouz2019advances} is designed to collaboratively train a global model on decentralized data across user clients while protecting data privacy.
FL emerged as an effective privacy-preserving solution of training (language) models, as rich text data are generated by users, which may contain sensitive and personal information. After \citet{fl} proposed to train on-device recurrent neural networks, FL has been widely used in various natural language processing applications and products, including next-word prediction \citep{gboard}, keyword spotting \citep{hard2020training}, and out-of-vocabulary word discovery \citep{chen2019federated}.

To further protect user privacy, Differential Privacy (DP)~\citep{dp1,dp2,dp3,mcmahan18learning} is introduced to provide formal privacy guarantees of models trained by federated learning. 
DP for deep learning explicitly adds random noise with bounded sensitivity to a training process (\textit{e.g.}, DP-SGD \citep{abadi2016deep}), ensuring a quantifiable similarity in output model distributions when the training dataset changes. 
When combining DP with FL, a variant of DP-SGD called DP-FedAvg \citep{mcmahan18learning}) is applied to guarantee user-level DP \citep{userdp}.
Current research primarily focuses on applying user-level DP to small on-device models with fewer than 10 million parameters \citep{mcmahan18learning,dpftrl,gboard2}. 
The model size is limited due to challenges such as significant DP noise required to preserve privacy \citep{dplm2} and the communication costs in cross-device FL.

Recent advances in large language models (LLMs) \citep{lamda,gpt2,gpt3,bert,t5}  have revolutionized natural language processing (NLP) and achieved unprecedented performance on various tasks such as text generation, machine translation, and sentiment analysis. 
However, their success comes at a cost of requiring massive amounts of computational resources, making them difficult to deploy on resource-constrained devices such as smartphones, tablets, or other edge devices. 
Additionally, there are concerns regarding the user privacy in various aspects such as memorizing personal information in training, and exposing private query in inference.

Recent work explore incorporating public information to improve privacy-utility trade-off in applying DP for (large) LMs \citep{dplm,dplm2}. Public data \citep{mirrordescent} or other side information \citep{li2022private} are also studied for (DP) FL.
In non-DP FL settings, \citet{nguyen2022begin} studies the effect of initializing from a pre-trained model.  
However, it is an open question on \textit{how to leverage the power of pre-trained LLMs to facilitate private FL for on-device LMs}.

In this work, we answer the question through systematic study aimed at enhancing private federated learning for on-device LMs with public pre-trained LMs.
Specifically, Our approach involves leveraging both public data and pre-trained LLMs to improve differentially private federated learning for on-device models by techniques of public pre-training and distillation.
Additionally, we propose a novel distribution matching algorithm, which is backed by theoretical analysis, to sample public data closely resembling the private data distribution, which significantly increases sample efficiency in public training.
Moreover, our extensive empirical results align with our theoretical predictions, further substantiating our approach.
Our work complements existing research by utilizing LLMs to improve public training through knowledge distillation for private cross-device federated learning, and achieve a strong privacy-utility trade-off with substantial improvements on sampling efficiency for public data. 
Our method points to a novel direction of efficiently enhancing private FL with public pretraining data and LLMs.

We summarize our \textbf{contributions} as follows: 
\begin{itemize}[leftmargin=0em,topsep=0pt,itemsep=0pt]
\item We focus on improving private federated learning for language modeling tasks and explore ways to leverage public data and pre-trained LLMs for tokenizers, training protocols, and data (sub)sampling. 
\item We conduct comprehensive studies and compare the use of Sentence Piece tokenizers from public LLM and unigram tokenizers from private corpus. We find that adopting public tokenizers from LLMs can not only prevent the potential privacy leakage from the private tokenizer vocabulary, but also lead to better learning utility with DP guarantees. 
\item For training protocol, we propose to leverage public LLM to teach private on-device LMs by knowledge distillation. We demonstrate that distilling public LLM to pre-train on-device LM can lead to more than $7\%$ accuracy improvement with tight privacy bound ($\varepsilon=1.77$). Moreover, it can achieve high data efficiency of using only $1\%$ of the public data compared to that in public pre-training without LLM, and attain better accuracy.
\item We further propose a novel distribution matching method that leverages both private on-device LMs and public LLMs to select public records close to private data distribution. 
We show that using $0.08\%$ of carefully sampled public data to train on-device LM can lead to comparable performance as public pre-training on-device LMs with the whole pre-training corpus. Moreover, it reduces the public training time from more than one week to a few hours.
Our method is grounded in theoretical analysis, which is corroborated by our extensive empirical results.
\end{itemize}

\section{Differentially Private Federated Learning for On-device LMs}
\label{sec:setup}
In this section, we walk through the preliminaries of differentially private federated learning of language models 
following the cross-device federated learning literature \citep{mcmahan18learning,kairouz2019advances,dpftrl}. 
We also introduce the experimental setup used throughout this paper.

\noindent \textbf{Cross-device Federated Learning.} 
\citet{fl} introduce federated learning to collaboratively train LMs for next-word prediction from decentralized user data on a large number of mobile devices without directly sharing the private data. A common training algorithm of federated learning is \texttt{FedAvg} \citep{fl}, where each client downloads the current model from the centralized server, computes an update by performing local computation on their dataset (\textit{e.g.}, running SGD) and sends the update back to the server. The server aggregates the updates across clients to update the global model and send the updated model back to local clients to achieve the goal of collaborative learning without directly accessing the training data on each user’s mobile device. 

In our experiments, we follow previous work \citep{dpftrl,mirrordescent,wu2022motley} and sample 100 clients in each training round. Each client uses a batch size of 16 for local training. We set the training rounds $T=1600$ in total. 

\noindent \textbf{User-level Differential Privacy.}
To further protect user privacy, Differential Privacy (DP)~\citep{dp1,dp2,dp3} was introduced to provide a formal privacy guarantee for federated learning.

\begin{definition}[$(\varepsilon, \delta)$-Differential Privacy]
\label{def:dp}
A randomized algorithm $\mathcal{M}$ with domain $\mathbb{N}^{|\mathcal{X}|}$ is $(\varepsilon,\delta)$-differentially private if for all $\mathcal{S} \subseteq \Range{\mathcal{M}}$ 
and for any adjacent datasets $D$ and $D'$:
\begin{equation*}
    \Pr[\mathcal{M}(D) \in \mathcal{S}] \leq \exp(\varepsilon)\Pr[\mathcal{M}(D')\in \mathcal{S}] + \delta.
\end{equation*}
\end{definition}

\Cref{def:dp} provides a formal definition of $(\varepsilon, \delta)$-DP by bounding the change in output distribution caused by a small input difference (or, adjacent datasets) for a randomized algorithm.  
In the FL setting, it is preferable to bound the output distribution caused by different users in order to protect the privacy of each client's whole dataset. 
Specifically, adjacent datasets of $D$ and $D'$ for user-level differential privacy \citep{userdp} are defined as: $D$
can be obtained from $D'$ by adding or subtracting all the records of a single user/client, which determines the unit of privacy guarantees. 

In our experiments, we use DP-FTRL \citep{dpftrl} for privacy accounting and private federated training, which can achieve strong privacy guarantee in practical FL scenarios \citep{xu2023gboard}. 
We use $\delta=10^{-6}$ and consider two $\varepsilon$ bounds: a tight privacy bound with $\varepsilon=1.77$ by using a large noise multiplier $m=8.83$, and a slightly loose privacy bound with $\varepsilon=18.71$ and noise multiplier $m=1.13$. We present more hyperparameter tuning details in Appendix \S\ref{app:exp}.

\noindent \textbf{On-device LMs.} 
Due to the limited memory constraints of mobile devices, on-device LMs are relatively small (usually less than 10M parameters). 
In our work, we focus on two types of on-device auto-regressive LMs: LSTM \citep{lstm} and transformers \citep{transformer}.
More model details can be found in Appendix \S\ref{app:pretrain}.

\noindent \textbf{Pre-trained LLMs.}
In addition to the on-device LMs trained on private datasets, this work also assumes that we have access to LLMs pre-trained on a large public corpus to aid private learning.
Specifically, we use LaMDA \citep{lamda} 2B throughout this work as an example, and conduct a systematic study of leveraging LLMs to help private training of on-device LMs.

\noindent \textbf{Datasets.} 
We focus on next word prediction task on the StackOverflow benchmark dataset (\citeyear{stackoverflow}) for private federated learning. 
Since StackOverflow is naturally keyed by users, each client in FL is a user in the Stack Overflow online forum. 
The examples of a client are sentences of questions and answers posted by a specific user.
We follow \citep{reddi2021adaptive,dpftrl} to construct a validation set of 10K samples, and a test set of 16.5M samples. 
Our evaluation metric is in-vocabulary next word (token) prediction accuracy, which is computed as the ratio of accurately predicted in-vocabulary words to the total number of words in the sequence (excluding OOV tokens). 

In addition to  StackOverflow as the (private) dataset, we  use {the realnews variant \texttt{c4/realnewslike}} of C4 dataset \citep{t5}, as the public dataset.
We analyzed the sources of the public C4 dataset and the Stackoverflow dataset for private training, and verified that there is no explicit overlap between public C4 dataset and the private StackOverflow dataset. More details can be found in Appendix \S\ref{app:verify}.

\section{Inspiration from LLMs}

The success of publicly pre-trained LLMs motivate us to have retrospective views on further improving private on-device LMs.
In this section, we explore inpiration from LLMs: the use of subword tokenizers and a large public corpus for pre-training.
We apply them to on-device LMs, and observe that both techniques bring significant performance improvement for private FL.

\begin{figure}[t]\small
    \centering
    \includegraphics[width=\linewidth]{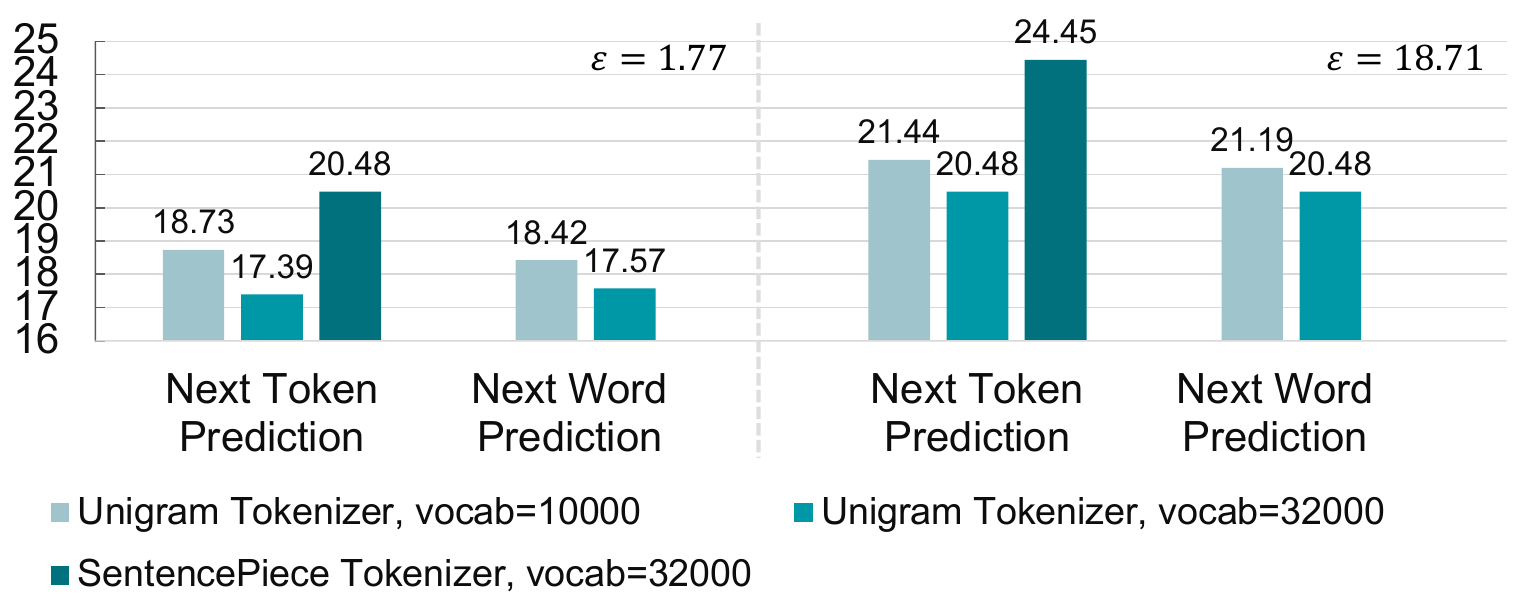}
      \caption{\small Next word (token) prediction accuracy for on-device LSTM with different tokenizers in the private FL.
      }
    \label{fig:tokenizer}
\end{figure}

\subsection{Using Public Tokenizer from LLMs}
\label{sec:tokenizer}

Tokenizer is an important module of LMs, which transforms natural languages into a sequence of predefined symbol sets (vocabulary).
Prior work in the literature of private FL of LMs \citep{mcmahan18learning,dpftrl,mirrordescent} use word-level unigram tokenizers potentially directly built from user data, which may need additional privacy budget \citep{ponomareva2022training,bagdasaryan2022training}. 

Recent LLMs adopt sub-word tokenizers \citep{kudo2018sentencepiece,bpe,wordpiece}, which mitigate most out-of-vocabulary (OOV) problems and yield state-of-the-art performance across different downstream tasks. 
This motivate us to replace the prior word-level unigram tokenizers with public sub-word tokenizers.
Specifically, we use SentencePiece tokenizer \citep{kudo2018sentencepiece} from  LaMDA.

To conduct comparison between unigram tokenizers and subword tokenizers for next word (token) prediction task, we convert the next word prediction accuracy into next token prediction accuracy. 
This conversion is achieved through splitting each word using the SentencePiece tokenizer. 
We consider all tokens within a word as accurate if the predicted word is correct.
We compare standard SentencePiece models (vocabulary size = $32K$) with unigram tokenizers that selects the top-$k$ frequent words from user data with $k=10K$ or $32K$ as vocabulary.

% results
We present the  private FL accuracy on the StackOverflow dataset  in Figure \ref{fig:tokenizer}. 
For the unigram tokenizer, using a larger vocabulary size in the DP setting can result in a slight performance drop, which can be different from the observation in non-DP settings~\citep{charles2022federated,xu2022adaptive}.
It is possible that the parameter increase of the embedding layer enlarges the effect of DP noise and hurts the final accuracy.
However, for next token prediction accuracy, 
although the public SentencePiece tokenizer from LaMDA also consists of $32K$ tokens,
it can significantly improve the private FL accuracy upon the unigram tokenizers, especially with smaller DP noise and $\varepsilon=18.71$.
We also observe that SentencePiece tokenizer finds no OOV tokens in the StackOverflow dataset, thus yielding the same high prediction accuracy with or without the OOV token. Therefore, we use SentencePiece tokenizer in the rest of this paper.

\subsection{Publicly pre-training for On-device LMs}
\label{sec:pretrain}

In addition to the use of subword tokenizers, LLMs benefit from pre-training on a large public corpus~\citep{li2022private,dplm}. In this section, we explore pre-training on-device LMs on public corpus to improve private federated learning.

\noindent \textbf{Pre-training Details.}  
We use the standard autoregressive language modeling loss $\mathcal{L}_{LM}$ to pre-train on-device LMs on the public C4 dataset, which takes around $1,400K$ steps (over a week of single GPU time) to process the entire dataset with the batch size of 512. 
We then use the publicly pre-trained checkpoint as the start point for private federated learning.
We leave more details in  \S\ref{app:pretrain}.

\begin{table}[t] \small
\centering
\begin{tabular}{l|cc|cc}
\toprule
                   & \multicolumn{2}{c|}{w/o pre-training} & \multicolumn{2}{c}{w/ pre-training} \\
                   \midrule
Rounds             & 0                    & 1600                & 0                   & 1600      \\
\midrule
$\varepsilon=1.77$          & \multirow{2}{*}{0.00}               & 20.48                          &       \multirow{2}{*}{16.94}               & 27.27              \\
$\varepsilon=18.71$ &                 & 24.45               &                  & 30.13\\
\bottomrule
\end{tabular}
\caption{\small Next Token Prediction Accuracy on the private StackOverflow dev set with or without public pre-training.}
\label{tab:pretrain}
\end{table}

\noindent \textbf{Results.} 
We present the next token prediction accuracy on the private StackOverflow dev set in Table \ref{tab:pretrain}. 
We observe that the accuracy on the private dataset significantly improves after pre-training for different different privacy budgets, shedding light on an effective way to boost private FL performance.
We also observe that after pre-training, it gives reasonable zero-shot accuracy on the private dataset even without private training (round=0). 

\begin{figure*}[tbh]\small
    \centering
\begin{subfigure}{.33\textwidth}
  \centering
    \includegraphics[width=\linewidth]{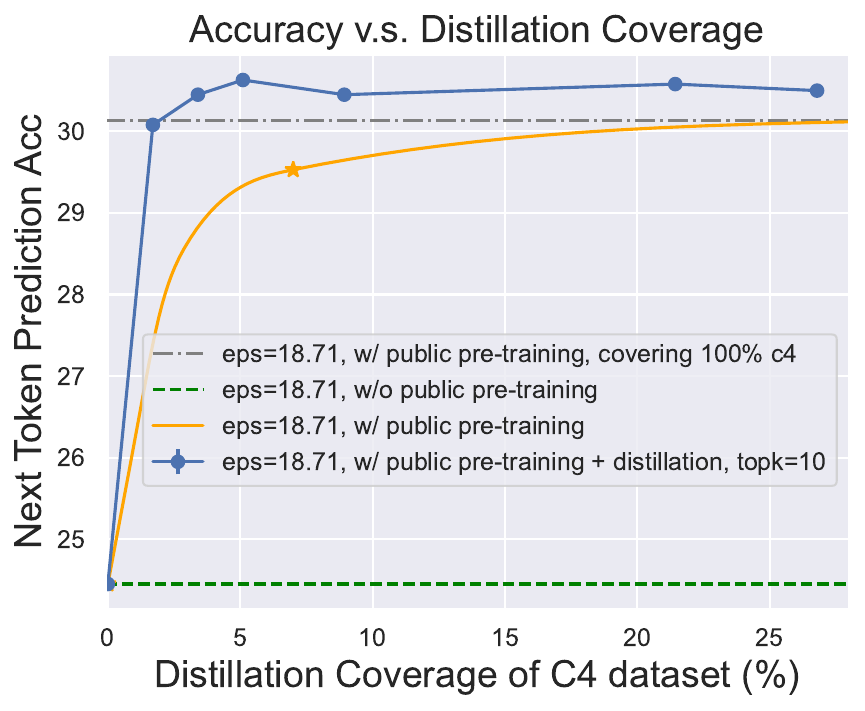}
  \caption{\footnotesize Acc. v.s. distillation steps ($\varepsilon=18.71$)}
  \label{fig:efficiency_small}
\end{subfigure}
\begin{subfigure}{.33\textwidth}
  \centering
    \includegraphics[width=\linewidth]{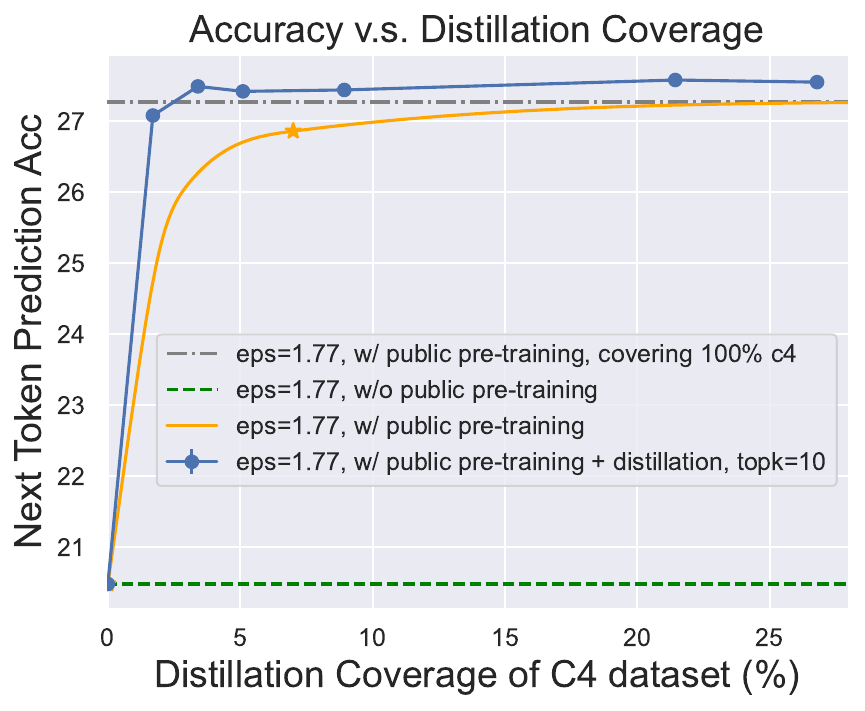}
  \caption{\small Acc. v.s. distillation steps  ($\varepsilon=1.77$)}
  \label{fig:efficiency_large}
\end{subfigure}
\begin{subfigure}{.33\textwidth}
  \centering
    \includegraphics[width=1.02\linewidth]{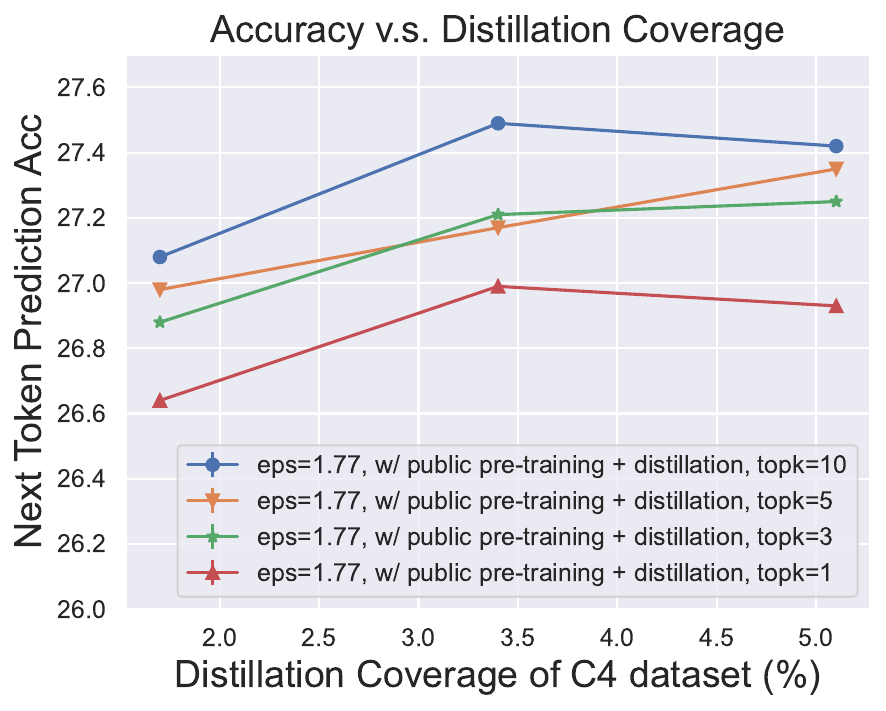}
  \caption{\small Acc. v.s. top-$k$ logits  ($\varepsilon=1.77$)}
  \label{fig:topk_small}
\end{subfigure}
    \caption{\small Ablation studies on how distillation steps and top-$k$ logits in distillation impact next token prediction accuracy (Acc.) of on-device LSTM models on the dev set of the private StackOverflow dataset.}
    \label{fig:efficiency}
\end{figure*}

\section{Distillation from Public LLM}
\label{sec:distill}

On one hand, the cost of public pre-training for on-device LMs is still expensive on a large public corpus (around a week of GPU time).
On the other hand, existing LLMs are well pre-trained and demonstrate promising performance across a variety of downstream tasks.
This motivates us to explore on whether we can leverage existing LLMs to improve the sample efficiency of pre-training on-device LMs. 
In this section, we answer the question above with systematic studies and show that we can improve the sample efficiency by using only $1\%$ of pre-training data and distillation from LLMs, achieving similar or even better performance than using $100\%$ of pretrianing data without distillation.

\subsection{Distillation Design}

Inspired by the literature of model compression \citep{sun2020mobilebert,jiao2019tinybert}, we use knowledge distillation to transfer the knowledge from trained LLMs into on-device LMs during pre-training.
The distillation pipeline contains the following two steps:
 
\noindent \textbf{Building a distillation corpus.} Given an input sequence from the public pre-training corpus, the LLM outputs the probability distribution over the vocabulary for next token prediction at each decoding step.
To construct a distillation corpus, we save the top-$k$ logits with $k$ nonzero entries $\boldsymbol{z}_T$ from the teacher LLM as a silver-label dataset.
In this way, the distillation corpus is model-agnostic, and thus can be applied to different variants of on-device LMs for pre-training. 
Moreover, selecting a reasonable top-$k$ for the logits can both help compress the distillation corpus to a moderate size and filter out noisy signals from tokens with low output probabilities.

\noindent \textbf{Public pre-training with distillation loss.}
Since we align the tokenizer of the on-device LM with the LLM to share the same vocabulary, we can align the output distribution of on-device LMs and LLMs by the cross-entropy loss.
Formally, for next token prediction task, given the output logits from student on-device LMs $\boldsymbol{z}_S$, the gold label from the pre-training corpus $\boldsymbol{y}$, and the logits from the distillation corpus of LLMs  $\boldsymbol{z}_T$, we add an additional knowledge distillation loss $\mathcal{L}_{KD} = \text{CE}(\boldsymbol{z}_S/t, \boldsymbol{z}_T/t)$ to the pre-training language modeling loss $\mathcal{L}_{LM} = \text{CE}(\boldsymbol{z}_S, \boldsymbol{y})$ as our public pre-training loss $\mathcal{L}_{\text{pub}} = \mathcal{L}_{LM} + \beta \mathcal{L}_{KD}$ 
where $t$ is the temperature. 
More distillation details are in \S\ref{app:distill}.

\subsection{Experimental Results}

After public pre-training with knowledge distillation, We use the checkpoints at different pre-training steps as the start point for private federated learning.
Our main results can be found in Table \ref{tab:dist_match}. We show that by using $1\%$ C4 dataset for pre-training with knowlegde distillation, we can significantly improve the sample efficiency without hurting but even improving the private FL accuracy for both LSTM and transformers, when compared with public pre-training on the whole C4 dataset. 
The sample efficiency improvement thus reduces the pre-training cost from one week to around one day, shedding light on a promising direction to improve the efficiency and utility of private FL.

\noindent \textbf{Ablation studies on distillation steps.}
To understand whether distillation for more epochs can help with private FL, we conduct a set of ablation studies on distillation steps given different privacy budgets as shown in Figure \ref{fig:efficiency_large} and \ref{fig:efficiency_small}. 
Specifically, we use the checkpoints at different distillation steps to initialize on-device LSTM and report the next word prediction accuracy after private FL at round 1600. 
We observe a consistent performance improvement when the distillation covers less than $5\%$ of the C4 dataset. 
But when we pre-train the LM for more epochs, the improvement becomes marginal. 
This suggests that teaching on-device LMs via LLMs can converge quickly within a few iterations. 

\noindent \textbf{Abaltion studies on top-$k$ logits.}
We take the top-$k$ logits of the LLM to construct our distillation datasets and pre-train the on-device LMs.
Here, we conduct an ablation study by pre-training different on-device LMs with different $k$ and evaluate how top-$k$ logits in distillation can impact the accuracy of private FL.
We present our empirical results in Figure \ref{fig:topk_small} and Appendix Figure \ref{fig:topk_large}.
We observe that pre-training with a larger $k$ is more helpful to achieve better downstream accuracy on private data. 
To have a reasonable trade-off between dataset size and pre-training performance, we use top-$k=10$ in all the following experiments.

\section{Distribution Matching}

In the previous section, we achieve compelling performance by employing LLM distillation using only $1\%$ of the \textit{randomly sampled} pre-training corpus.
Now we further investigate the possibility of improving sample efficiency by selectively identifying public samples that align with the distribution of private samples.
To this end, we propose a novel distribution matching method to sample public records for pre-training with a novel theoretical analysis jointly considering public-private distribution shift and DP mechanism.
We demonstrate that by carefully selected $0.08\%$ of public samples, we can pre-train on-device LMs that perform as well as using $1\%$ of public samples with distillation.
This approach significantly improves sample efficiency, providing an additional knob of using public pre-training for private on-device models.

\subsection{Algorithm}

We hypothesize two principles to sample public records to match the private distribution: 
$(i)$ the probability of the public sample $x$ on the private data distribution $p_{\text{priv}}(x)$ is high, which can be approximated by the prediction of the on-device LMs trained on the private dataset;
$(ii)$ the probability of a public sample $x$ on the public data distribution $p_{\text{pub}}(x)$ is also high, as we expect those samples are easy-to-learn \citep{swayamdipta2020dataset} and of high data quality in the public corpus.
The probability $p_{\text{pub}}(x)$ can be approximated by the public pre-trained LLMs.

To verify our hypothesis, we visualize the perplexity (PPL) distribution of public samples and private samples evaluated by both a privately fine-tuned on-device LM and a public pre-trained LLM in Figure \ref{fig:dist_matching}.
To have an ``oracle'' on-device LM that well captures the private data distribution, we fine-tune it on the private data without DP noise to overfit the private data distribution.
We randomly sample 10k records from the public dataset and private dataset, respectively.
We observe that the \textcolor{blue}{private dataset}  mostly concentrates on the regime with low PPL evaluated by the public and private LMs, whereas the \textcolor{orange}{public dataset} is more diverse and distributed across a broader range of PPL values. 
The distribution visualization confirms our hypothesis to select \textcolor{orange}{public samples} from the lower left corner, which correspond to samples with high probabilities  $p_{\text{pub}}(x)$ and  $p_{\text{priv}}(x)$ on public and private data distribution (\textit{i.e.}, low perplexity evaluated by public and pirvate LMs).

\begin{figure}[t]
    \centering
    \includegraphics[width=\linewidth]{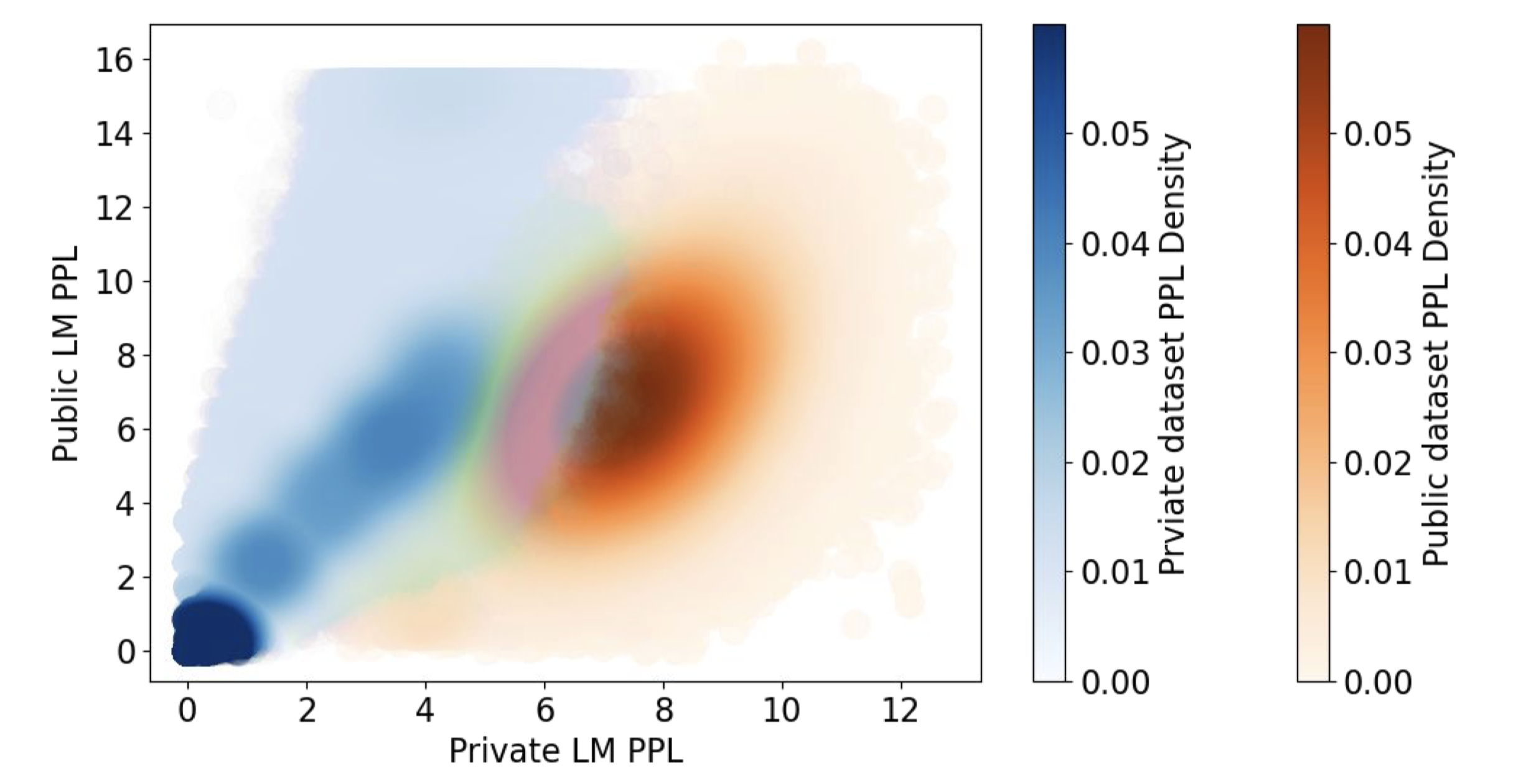}
    \caption{\small Visualization of PPL distribution of the private and public datasets evaluated by the private on-device LM and the public LLM. The private dataset exhibits a concentration of low PPL values, whereas the public corpus is dispersed across a broader range of PPL values, with a higher average PPL.}
    \label{fig:dist_matching}
\end{figure}

In practice, we do not have an ``oracle'' on-device LM trained on private data for distribution match. 
Instead, we propose to fine-tune an on-device LM with DP for certain rounds $T' < T$ before consuming all the privacy budgets, and then use the checkpoint at round $T'$ with DP guarantee to approximate $p_\text{priv}(x)$ and perform distribution matching to sample public records. This post-processing based on a DP checkpoint will not incur any additional privacy cost.
Thereafter, we can use the sampled \textit{public} records to further train the private checkpoint at round $T'$, as a way for efficient public (pre-)training. 
Following the strategy in \S\ref{sec:distill}, we also employ the distillation loss to better train the on-device LM with carefully sampled public records to further enhance the sample efficiency. 
Lastly, we use the remaining privacy budgets to fine-tune the on-device LM until reaching round $T$, and evaluate its next token prediction accuracy at the dev and test sets. 
We term the paradigm of two-stage private learning combined with public training as ``public mid-training''.
This approach differs from ``public pre-training'', which involves public pre-training prior to private FL.
We present the distribution matching protocol in Algorithm \ref{algo}.

\begin{algorithm}[t]
\begin{small}
  \caption{\small Leveraging LLMs for distribution matching and public training in private federated learning. } \label{algo}
  \begin{flushleft} 
  \textbf{Input:} Public pre-training corpus $D$, private corpus $D^*$, sampling rate $q$, private fine-tuning rounds $T$, first-stage fine-tuning rounds $T' < T$ for distribution matching, a public pre-trained LLM \\
  \textbf{Output:} Private on-device LM with DP guarantee
  \end{flushleft} 
  \begin{algorithmic}[1]
  \State Randomly initialize an on-device LM;
  \State \emph{// \textcircled{1} First-stage private federated learning}
  \State Use DP-FTRL to train the on-device LM for rounds $T'$; 
  \For{each $x \in D$}
    \State \emph{// \textcircled{2} Probability evaluation}
    \State Compute the average (token) log prob $\log p_{\text{priv}}(x)$ given the privately ﬁne-tuned LM at round $T'$;
    \State Compute the average (token) log prob $\log p_{\text{pub}}(x)$ given a publicly pre-trained LLM ;
  \EndFor 
  \State \emph{// \textcircled{3} Distribtion matching}
  \State Sort $D$ based on $\log p_{\text{priv}}(x)$ + $\log p_{\text{pub}}(x)$
  \State Sample a subset of $D$ as $D'$ with top $\log p_{\text{priv}}(x)$ + $\log p_{\text{pub}}(x)$ values, such that $|D'| = q |D|$.
  \State \emph{// \textcircled{4} Public mid-training with LLM distillation}
  \State Train the on-device LM with the loss $\mathcal{L}_{\text{pub}}$ on $D'$
  \State \emph{// \textcircled{5} Second-stage private federated learning}
  \State Use DP-FTRL to train the on-device LM for the remaining rounds of $T-T'$
  \State \textbf{return} On-device LM with DP guarantee
  \end{algorithmic}
  \end{small}
\end{algorithm}

\begin{table*}[t]\small
    \resizebox{1.0\textwidth}{!}{
    \begin{tabular}{@{}l|l|c|c|cc|cc@{}}
    \toprule
    & \multirow{2}{*}{$q$ ($\%$ of} & \multirow{2}{*}{LLM} & \multirow{2}{*}{Distribution} & \multicolumn{2}{c|}{Accuracy (LSTM)} & \multicolumn{2}{c}{Accuracy (Transformer)}\\ \cmidrule{5-8}
                                                             &   Public Data)   &   Distillation   &   Matching   &   $\varepsilon$=1.77   &  $\varepsilon$=18.71        & $\varepsilon$=1.77     &  $\varepsilon$=18.71    \\ \midrule
    \textbf{No Public Training}                              &   $0\%$         &                  &           &    $20.68_{\pm 0.04}$             &     $28.87_{\pm 0.04}$                 &    $23.98_{\pm 0.15}$             &     $28.29_{\pm 0.06}$             \\ \midrule
    \textbf{Pre-training  w/ public data ($T'=0$)}                    &   $100\%$       &                  &           &    $28.01_{\pm 0.26}$             &     $30.70_{\pm 0.01}$                 &    $\textbf{28.05}_{\pm 0.02}$             &     $30.10_{\pm 0.00}$             \\
    $\quad$ \textbf{$\cdot$ LLM Distillation (100k steps)}    &   $1\%$         &     \cmark       &           &    $\textbf{28.68}_{\pm 0.09}$             &     $\textbf{31.13}_{\pm 0.03}$                 &    $27.75_{\pm 0.06}$             &     $\textbf{30.19}_{\pm 0.01}$             \\ 
    $\quad$ \textbf{$\cdot$ LLM Distillation (8k steps)}      &   $0.08\%$      &     \cmark       &           &    $26.18_{\pm 0.04}$             &     $29.53_{\pm 0.10}$                 &    $25.31_{\pm 0.08}$             &     $29.36_{\pm 0.12}$             \\ \midrule
    \textbf{Mid-training w/ public data ($T'=T/2$)}                     &   $0.08\%$      &                  &           &    $26.67_{\pm 0.06}$             &     $29.76_{\pm 0.03}$                 &    $25.83_{\pm 0.03}$             &     $29.15_{\pm 0.01}$             \\
    $\quad$ \textbf{$\cdot$ LLM Distillation (8k steps)}      &   $0.08\%$      &     \cmark       &           &    $27.01_{\pm 0.03}$             &     $30.18_{\pm 0.06}$                 &    $26.04_{\pm 0.12}$             &     $29.47_{\pm 0.05}$             \\
    $\quad$ $\quad$ \textbf{+ Distribution Matching}                    &   $0.08\%$      &     \cmark       &  \cmark   &    $\textbf{28.01}_{\pm 0.08}$             &     $\textbf{30.63}_{\pm 0.02}$                 &    $\textbf{27.17}_{\pm 0.03}$             &     $\textbf{29.83}_{\pm 0.01}$            \\
    \bottomrule
    \end{tabular}
    }
    \caption{\small Summary of techniques to improve downstream stream next token \textbf{prediction accuracy} and \textbf{sample efficiency} for on-device LSTM and transformer model evaluated on the StackOverflow test set.}
    \label{tab:dist_match}
\end{table*}

\subsection{Theoretical Analysis}
In this section, we provide the theoretical analysis of our distribution matching protocol to present the \textit{intuition} behind our selection hypothesis. 
In essence, the goal of our distribution matching algorithm is to have a good estimator for the private distribution.
However, characterizing the distribution shift in the context of differential privacy is a challenging problem, in that the private models are trained with DP noise, which can yield an inaccurate estimation of private data distribution, and thus add the complexity to our analysis.

\noindent\textbf{Problem Setup.} 
Define the text data domain as $\gX$. Denote $\ell_{\text{pub}}:\gX\to \sR$ as the log-density function of the public data distribution (i.e., $\ell_{\text{pub}}(x)=\log p_{\text{pub}}(x)$ where $p_{\text{pub}}(x)$ is the public data density estimated by public LLMs), and $\ell_{\text{priv}}$ as the \textit{accurate} log-density function of the private data distribution (i.e., $\ell_{\text{priv}}(x)=\log p_{\text{true priv}}(x)$ where $p_{\text{true priv}}(x)$ is the true private data density). However, due to limited private data sampled from the true private data distribution and DP noise injected in the private FL, we can only obtain an inaccurate estimation $\hat \ell_{\text{priv}}=\log p_{\text{priv}}(x)$ of the true private log-density $\ell_{\text{priv}}$, where $p_{\text{priv}}(x)$ is the private data density estimated by private on-device LMs. Note that we use the hat notation $\hat \ell_{\text{priv}}$ to denote that it is an estimation of the true private log-density $\ell_{\text{priv}}$.

We can view the estimation $\hat \ell_{\text{priv}}$ is a random variable where the randomness comes from: (i) that the private dataset we have is sampled from the private data distribution; and (ii) the randomness in the algorithm of obtaining $\hat \ell_{\text{priv}}$ based on the private dataset, e.g., differential privacy.
Following previous work~\citep{jiang2023federated}, we make a standard assumption. 
We assume the estimated private data log-density function is an unbiased estimator, i.e., $\E[\hat \ell_{\text{priv}}]=\ell_{\text{priv}}$.
Since $\ell_{\text{pub}}$ may not be ideal because of public-private domain shift, and $\hat\ell_{\text{priv}}$ may mot be ideal because of its DP noise, $\ell_{\text{pub}}$ and $\hat \ell_{\text{priv}}$ are neither good estimators for $\ell_{\text{priv}}$. 
\textit{Can we leverage both of the information and form a function $\hat h:\gX\to \sR$ that combines $\ell_{\text{pub}}$ and $\hat\ell_{\text{priv}}$ such that $\hat h$ is a good estimator for $\ell_{\text{priv}}$?}
In the following analysis, we choose $\hat h=\frac{1}{2}\ell_{\text{pub}}+\frac{1}{2}\hat\ell_{\text{priv}}$ and analyze  when and why it can be a better estimator to the true private log-density $\ell_{\text{priv}}$ than $\ell_{\text{pub}}$ and $\hat\ell_{\text{priv}}$.

We need some mathematical tools to define what does it mean to be ``better''. Concretely, we need a metric to measure the distance between functions. This can be done by having an inner product $\langle \cdot, \cdot \rangle$ in the function space of $\gH=\{f:\gX\to \sR\}$, and hence the norm in the function space $\gH$ is  $\|f\|=\sqrt{\langle f,f\rangle}$ for $\forall f\in \gH$. Our analysis holds with \textit{any} choice of the inner product as long as it does not make the log-densities norm infinite. We discuss a concrete choice of the inner product and its relation to the KL divergence in Appendix~\S\ref{app:theory}.

With the norm as a ``ruler'', we are able to define the following key quantities that formally characterize the setting.
\begin{enumerate}[leftmargin=1.3em,topsep=1pt,noitemsep]
    \item \textbf{Public-Private Domain Distance.} Let $d_{\text{pub, priv}}=\|\ell_{\text{pub}}-\ell_{\text{priv}}\|$ denote the distance between the public data log-density $\ell_{\text{pub}}$ and the true private log-density $\ell_{\text{priv}}$. 
    \item \textbf{Private Domain Randomness.} Let $\sigma_{\text{priv}}^2=\E[\|\hat\ell_{\text{priv}}-\ell_{\text{priv}}\|^2]$ denote the randomness of the estimated private log-density, i.e., the quality of the estimated private log-density $\hat \ell_{\text{priv}}$
\end{enumerate} 
The above definitions are important because the quality of a private log-density estimator would depend on the public-private domain shift and the private domain randomness as we show next. 

\begin{theorem}
\label{thm:main}
Let $\epsilon(\hat f)=\E[\|\hat f-\ell_{\text{priv}}\|^2]$ characterise how good $\hat f$ is as an estimator of the true private data log-density $\ell_{\text{priv}}$ for any random function $\hat f\in \gH$.
Consider the following three quantities: 
\begin{enumerate}[leftmargin=1.3em,topsep=1pt,noitemsep]
    \item $\epsilon(\ell_{\text{pub}})$ characterizing the error of the public log-density function $\ell_{\text{pub}}$ to approximate  $\ell_{\text{priv}}$
    \item $\epsilon(\hat\ell_{\text{priv}})$ depicting the error of the noisy private log-density function $\hat\ell_{\text{priv}}$ to approximate  $\ell_{\text{priv}}$
    \item $\epsilon(\hat h)$ characterizing the error of $\hat h=\frac{1}{2}\ell_{\text{pub}}+\frac{1}{2}\hat\ell_{\text{priv}}$ to approximate  $\ell_{\text{priv}}$.
\end{enumerate} 
Then,
\begin{align}
    \epsilon(\ell_{\text{pub}})&=d^2_{\text{pub, priv}}\\
    \epsilon(\hat\ell_{\text{priv}})&=\sigma_{\text{priv}}^2\\
    \epsilon(\hat h)&=\frac{1}{4}d^2_{\text{pub, priv}}+\frac{1}{4}\sigma_{\text{priv}}^2
\end{align}
\end{theorem}

\noindent\textbf{Interpretation $\quad$}
\Cref{thm:main} implies that:
\begin{itemize}[leftmargin=0.5em,topsep=0pt,noitemsep]
    \item $\epsilon(\hat h)\leq  \frac{1}{2} \max\{\epsilon(\ell_{\text{pub}}), \epsilon(\hat\ell_{\text{priv}})\}$. 
    \item $\epsilon(\hat h)\leq \min\{\epsilon(\ell_{\text{pub}}), \epsilon(\hat\ell_{\text{priv}})\} $ if $\tfrac{1}{3}\leq \frac{d^2_{\text{pub, priv}}}{\sigma_{\text{priv}}^2} \leq 3$. 
\end{itemize}
Combining the above, we have the following conclusion: recall $\hat h=\frac{1}{2}\ell_{\text{pub}}+\frac{1}{2}\hat\ell_{\text{priv}}=\tfrac{1}{2}\log (p_{\text{\text{pub}}}(x) p_{\text{\text{priv}}}(x))$. We can expect that $\hat h$ is better than either $\ell_{\text{pub}}$ or $\hat\ell_{\text{priv}}$ for any settings. Moreover, we can expect $\hat h$ to be better than  both $\ell_{\text{pub}}$ and $\hat\ell_{\text{priv}}$ if (i) there is a domain shift between the public-private domain; and (ii) our estimated private log-density $\hat \ell_{\text{priv}} $ is noisy in an extent comparable to the domain shift. 
We leave the full proof and additional discussion in Appendix \ref{app:theory}.

\begin{table}[t]\small
    \resizebox{1.0\linewidth}{!}{
        \begin{tabular}{l|ll|ll}
            \toprule
& \multicolumn{2}{c|}{LSTM} & \multicolumn{2}{c}{Transformer}\\ 
      &   $\varepsilon$=1.77             &  $\varepsilon$=18.71                   & $\varepsilon$=1.77                &     $\varepsilon$=18.71    \\ \midrule
w/ $p_\text{pub}(x)$      &   $\textbf{28.01}_{\pm 0.08}$             &     $\textbf{30.63}_{\pm 0.02}$                 &    $\textbf{27.17}_{\pm 0.03}$             &     $29.83_{\pm 0.01}$            \\
w/o $p_\text{pub}(x)$     &   $27.77_{\pm 0.05}$             &     $30.56_{\pm 0.06}$                 &    $26.70_{\pm 0.04}$             &     $\textbf{30.18}_{\pm 0.05}$            \\
\bottomrule
\end{tabular}
}
\caption{\small Ablation studies on the use of public LLM for distribution matching evaluated on the StackOverflow test set.}
\label{tab:p_pub}
\end{table}

\subsection{Experimental Results}

\noindent \textbf{Experimental Setup.} We set $T'= T/2 = 800$ rounds for the first-stage private federated learning.
We use $q=0.08\%$ of the whole pre-training corpus for public training, which reduces the public training time from more than 1 weeks to a few hours with a single GPU. 
For the public mid-training setting, we also evaluate how LLM distillation and distribution matching can impact the private FL accuracy, respectively. 
We run all the experimental settings for three times and report
the average and standard deviation of test accuracy
on the private StackOverflow dataset.

We present the results of on-device LSTM and transformers in Table \ref{tab:dist_match}. 
In the pre-training setting $(T'=0)$, we show that we cannot further improve the sample efficiency from $1\%$ to $0.08\%$  with LLM distillation improves the sample efficiency, as the final accuracy after private FL significantly decreases. 
In comparison, in the mid-training setting ($T'=T/2$), using LLM distillation on the $0.08\%$ of randomly sampled pre-training corpus already gives better performance than pre-training.
Moreover, with distribution matching to carefully sample public data, we further improve the private FL accuracy, attaining comparable performance to the setting using the whole public corpus for pre-training.

\noindent \textbf{Ablation studies on $p_\text{pub}(x)$.}
Our distribution matching algorithm leverages both on-device LM and LLM to sample data close to the private distribution. 
To understand how the use of LLM ($p_{\text{pub}}(x)$) impact the sampling quality, we conduct an ablation study to sample a subset of $D'$ based on top $\log p_\text{priv}(x)$ values alone instead of  $\log p_\text{priv}(x) + \log p_\text{pub}(x)$.
We use the $p_\text{priv}$-sampled $D'$ for public mid-training and report the test accuracy of three runs for both on-device LSTM and transformers given different privacy budgets in Table \ref{tab:p_pub}.
The experimental findings corroborate our theoretical analysis. Specifically, when on-device language models (LMs) are trained with high noise levels ($\varepsilon=1.77$), we find that a combined utilization of both on-device LMs and LLMs consistently yields superior performance. This is because the estimated private log-density $\hat \ell_{\text{priv}}$ is noisy to a degree comparable to the domain shift, making $\hat h$ a more reliable estimator than $\hat \ell_{\text{priv}}$.
Conversely, when on-device LMs are trained with low noise ($\varepsilon=18.71$), the performance difference between models with and without $p_{\text{pub}}$ is negligible. This indicates that the noise introduced by differentially private (DP) training is not as significant as the distribution shift, allowing $\hat \ell_{\text{priv}}$ to serve as a good estimator.

\begin{table}[]\small
    % \resizebox{1.0\linewidth}{!}
    \centering
    {
        \begin{tabular}{l|lllll}
            \toprule
      $T'$                        &   $0$             &  $400$                   & $800$                &     $1200$ & $1600$    \\ \midrule
      $\varepsilon$=1.77  & $25.41$ & $27.08$   &  $\textbf{27.73}$ & $26.40$   & $18.40$           \\
      $\varepsilon$=18.71 & $28.38$ & $30.07$   &  $\textbf{30.37}$ & $29.45$  & $19.34$     \\
\bottomrule
\end{tabular}
}
\caption{\small Ablation studies on the timing ($T'$) of distribution matching for mid-point public training on on-device LSTM evaluated the StackOverflow dev set.}
\label{tab:t_prime}
\end{table}

\noindent \textbf{Ablation studies on $T'$.}
$T'$ separates two-stage private federated learning and determines the timing for distribution matching and public training. 
In this ablation study, we evaluate the dev set accuracy of on-device LSTM given different $T'$ and privacy budgets, as shown in Table \ref{tab:t_prime} and Appendix Table \ref{tab:t_prime_random_sampling}. 
From the table, we can see that the on-device LSTM achieves the best private FL accuracy given $T' = T/2 = 800$.
We think the reasons are as follows:
when $T'=0$, we cannot perform distribution matching as the on-device LM is not trained on the private dataset yet, and thus we can only use the randomly sampled data for pre-training;
when $T'=400$, the on-device LM could not be well trained on the private data distribution, thus yielding worse distribution matching quality;
when $T'=1200$ and $T'=1600$, the private on-device LM is biased towards the public data distribution due to public training, thus giving worse private FL accuracy. 
As a result, we use $T'=800$ in our main experiments, as it balances the private federated training and public training to have satisfactory distribution matching capabilities without biasing too much towards the public data distribution. 

\section{Conclusion}
In this work, we propose to improve private federated learning by using LLMs in public training. 
We leverage LLMs to aid public training of on-device LMs via distribution matching to sample public data close to private data distribution, which further improves the effectiveness and efficiency of public training, demonstrating strong private learning accuracy while minimizing the need for large amounts of public training data. 
Our work sheds light on a promising direction to improve private federated learning with public LLMs.

\subsection*{Acknowledgement}
We gratefully thank the anonymous reviewers and meta-reviewers for their constructive feedback. The authors thank the early feedback from Yanxiang Zhang. BL acknowledges support from the National Science Foundation under grant No. 2046726, No. 2229876, DARPA GARD, the National Aeronautics and Space Administration (NASA) under grant No. 80NSSC20M0229, and Alfred P. Sloan Fellowship.

\subsection*{Limitations}
This work has paved the way for enhancing the utility of differentially private on-device FL models, using large-scale public data and LLMs, but we also acknowledge the following limitations:

\begin{itemize}[leftmargin=1.3em,topsep=1pt,noitemsep]
\item \textbf{Data Distribution Matching}: The proposed distribution matching algorithm aims to sample public data close to the private data distribution. 
The choice of $\hat h$ can be data dependent and a weighted combination of $\ell_{\text{pub}}$ and $\hat\ell_{\text{priv}}$, \textit{i.e.}, $\hat h = (1-\beta)\ell_{\text{pub}}+\beta \hat\ell_{\text{priv}}$ where $\beta\in [0,1]$, as mentioned in Appendix \S\ref{app:extension}.
In practice, the optimal $\beta$ can be an important hyper-parameter to tune the distribution matching algorithm. 
Our work mainly leverages $\hat h=\frac{1}{2}\ell_{\text{pub}}+\frac{1}{2}\hat\ell_{\text{priv}}$ to analyze when and why a better estimator to the true private log-density $\ell_{\text{priv}}$ than $\ell_{\text{pub}}$ and $\hat\ell_{\text{priv}}$. We leave it as important future direction to get the optimal $\beta$ theoretically and empirically.
\item \textbf{Computational Resources}: The use of large-scale public data and LLMs can improve the privacy-utility trade-off in DP FL models, but this often comes at the cost of computational resources. Our work mainly focuses on LaMDA 2B as an example of LLM due to the lack of computational resources. While our main focus does not lie in the knowledge distillation, we leave it as future work to extend the size of LLMs in public pre-training.
\end{itemize}

\bibliography{anthology,custom}

\clearpage
\appendix
\onecolumn
\section{Additional Related Work}

\paragraph{Private Federated Learning in On-device NLP}
Federated learning is designed to collaboratively training NLP models without sharing sensitive user data to protect user privacy. 
Given relatively small model sizes, state-of-the-art differentially private (DP) learning algorithms \citep{mcmahan18learning,dpftrl} have enabled on-device LMs to achieve strong downstream task utility with reasonable user-level differentially privacy guarantee \citep{userdp}. 
The success of private FL has also led to real-world applications such as GBoard, which uses on-device LMs for next word prediction \citep{gboard,gboard2}.
Recent advances in DP optimization \citep{dpftrl} further improves upon the state-of-the-art DP-SGD algorithm \citep{abadi2016deep}, providing a practical tool to analyze privacy bound for federated learning. 

\paragraph{Privacy-preserving Large NLP Models} 
Scaling up LMs with more data and parameters has significantly improved performance and achieved great success in a variety of NLP tasks. 
Moreover, recent studies show that LLM has great potential in private learning. 
For example, \citet{kerrigan-etal-2020-differentially} show that public pre-training is helpful for downstream DP fine-tuning. 
Follow-up studies argue that large pre-trained LMs can be strong differentially private learners with parameter-efficient fine-tuning \citep{dplm,bu2022differentially} or full model fine-tuning \citep{dplm2}, narrowing the gap between non-private training and private training. 
\citet{Ganesh2023WhyIP} also provide theoretical groundings on the necessity of involving public training into private learning.
Motivated by the recent success of LLMs, our work performs comprehensive studies on how to use public data and existing LLMs to help private training of cross-device FL models.

\paragraph{Model Compression for Pre-trained LMs}
One promising approach to address the resource limitations of LLMs is to compress them into smaller models through various techniques such as knowledge distillation \citep{jiao2019tinybert,sun2020mobilebert,wang2020minilm}, or pruning \citep{elbayad2019depthadaptive,gordon-etal-2020-compressing}. While these techniques have demonstrated success in reducing the size of pre-trained LMs, most resulting models are still too large (with over 10 million parameters) to be effectively deployed on resource-constrained devices.
In our work, we also explore the use of knowledge distillation in public training, but with a primary focus on leveraging LLMs to improve sample efficiency in pre-training on-device LMs. We aim to improve the private FL performance of on-device LMs while minimizing the need for large amounts of training data. 
We recognize that private federated learning can further benefit from advanced model compression techniques, and we leave this as a promising and orthogonal future direction for research in this area.

\section{Experimental Setup Details}
\label{app:setup}

\subsection{Verification of Non-overlap between C4 and StackOverflow Datasets}
\label{app:verify}

StackOverflow contains 342K clients for training with 135.8M examples.
In this section, we detail the method used to verify that there is no explicit overlap between the public C4 dataset and the private StackOverflow dataset utilized in our study.

We explored C4 which has multiple variants\footnote{https://www.tensorflow.org/datasets/catalog/c4}: \texttt{c4/en, c4/realnewslike, and c4/webtextlike}. 

To verify this hypothesis, we conducted a rigorous comparison of these two datasets and its variants. 
Specifically, we compared the unique identifiers (e.g., URL for webpages in the C4 dataset, and post ID for StackOverflow posts) between the two datasets. 

No matching identifiers were found between the \texttt{c4/realnewslike} and the StackOverflow dataset.
Thus we use the \texttt{c4/realnewslike} variant as our public pretraining corpus throughout the experiment. 

Through this comprehensive comparison, we have confirmed that there is no explicit overlap between the public C4 dataset and the private StackOverflow dataset. This conclusion is critical to our study as it ensures that the integrity and privacy-preserving conditions of our experiment are maintained.

\subsection{Pretraining Details}
\label{app:pretrain}
In this section, we outline the detailed procedures followed during the pretraining phase of our experiments. The pretraining phase consisted of the following steps:

\begin{enumerate}
\item \textbf{Data Preparation:} We tokenized both the C4 and StackOverflow datasets using the SentencePiece tokenizer, as described in the main text. The vocabulary size was set to $32K$ for both datasets.

\item \textbf{Model Architecture:} We follow previous work \citep{wang2021fieldguide,mirrordescent,dpftrl,wu2022motley} and use one-layer LSTM and transformer. Both LSTM and transformer has a hidden size of 670 and embedding size of 96.

\item \textbf{Training Procedure:} We trained the model using a standard autoregressive LM loss for next token prediction. 

\item \textbf{Training Hyperparameters:} We employed the Adam optimizer with a learning rate of 1e-3, a batch size of 512, and a maximum sequence length of 20 tokens. We also used gradient clipping to prevent exploding gradients. The model was pretrained for $1400K$ steps on the C4 dataset to cover the whole C4 pretraining corpus. 
\end{enumerate}

After pretraining, the model was then fine-tuned on the downstream task using federated learning with differential privacy. Further details regarding the fine-tuning process can be found in the relevant sections of the main text. 
We show that the pretraining procedure can significantly improve the model's robust performance in the downstream task performance.

\subsection{Distillation Details}
\label{app:distill}

In this section, we delineate the specifics of our distillation process during the pretraining phase of our on-device LM.  The pretraining procedure with distillation is mostly the same as details outlined in \ref{app:pretrain} with slight hyper-parameter differences.

We set the temparature $t=1$ and top-$k=10$ to extract the logits $\boldsymbol{z}_T$ from teacher LLM. We use grid search to tune the best hyper-parameter $\beta \in \{1e-1, 1e-2, 1e-3\}$ and follow the same pre-training schedules as \S\ref{sec:pretrain} but with a smaller batch size of 128 due to memory constraints.

\section{Additional Experimental Results}
\label{app:exp}
\paragraph{Hyper-parameter Tuning for Federated Learning}
Federated learning involves numerous hyperparameters, which is crucial for our experiment. Our hyper-parameter tuning strategy follows \citet{xu2022learning}.

Throughout our experiments, we fix the number of total rounds $T=1600$. 
In each round, we select 100 clients from the shuffled pool for DP-FTRL, ensuring that the clients are disjoint across rounds.  Within each client, we fix the number of local epochs to one and set the batch size to 16. We also impose a constraint on the maximum number of samples on each client, limiting it to 256. 

We tune the server learning rate, client learning rate and clip norm for a certain given a noise multiplier. 
Specifically, we use grid search and tune the server learning rate from $\{0.05,0.1,0.2,0.5,1,2\}$, the client learning rate from $\{0.01, 0.02, 0.05, 0.1, 0.2, 0.5\}$.
We use the adaptive clipping technique in \citep{andrew2019differentially,xu2023gboard} to help determine the clip norm, which in most of our experiments falls into $\{0.1,0.3,0.4,1\}$.

\paragraph{Abaltion studies on top-$k$ logits}
We take the top-$k$ logits of the LLM to construct our distillation datasets and pre-train the on-device LMs.
Here, we conduct an ablation study by pre-training different on-device LMs with different $k$ and evaluate how top-$k$ logits in distillation can impact the accuracy of private FL.
We present our empirical results in Figure \ref{fig:topk_small} and Appendix Figure \ref{fig:topk_large}.
We observe that pre-training with a larger $k$ is more helpful to achieve better downstream accuracy on private data. 
To have a reasonable trade-off between dataset size and pre-training performance, we use top-$k=10$ in all the following experiments.

\begin{figure}[htp!]\small
    \centering
    \includegraphics[width=0.5\linewidth]{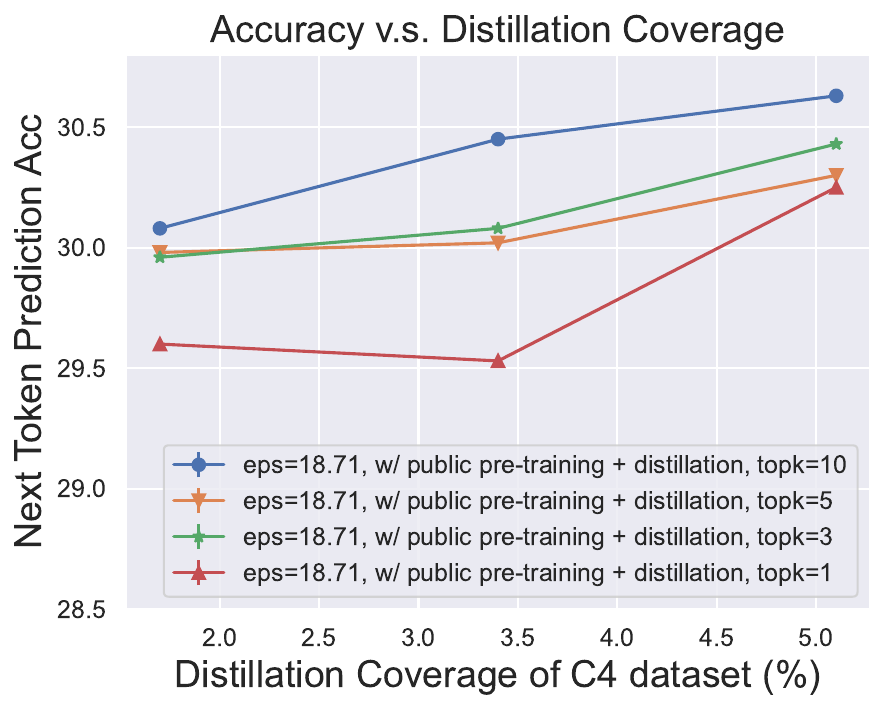}
    \caption{\small Ablation studies on how distillation steps and top-$k$ logits in distillation impact next token prediction accuracy (Acc.) of on-device LSTM models on the private StackOverflow dataset.}
      \label{fig:topk_large}
\end{figure}

\paragraph{Ablation studies on the timing $T'$ for mid-training}
$T'$ separates two-stage private federated learning and determines the timing for distribution matching and public training. 
In this ablation study, we evaluate the dev set accuracy of on-device LSTM given different $T'$ and privacy budgets, as shown in Table \ref{tab:t_prime} and Appendix Table \ref{tab:t_prime_random_sampling}. 
From the table, we can see that the on-device LSTM achieves the best private FL accuracy given $T' = T/2 = 800$.
We think the reasons are as follows:
when $T'=0$, we cannot perform distribution matching as the on-device LM is not trained on the private dataset yet, and thus we can only use the randomly sampled data for pre-training;
when $T'=400$, the on-device LM could not be well trained on the private data distribution, thus yielding worse distribution matching quality;
when $T'=1200$ and $T'=1600$, the private on-device LM is biased towards the public data distribution due to public training, thus giving worse private FL accuracy. 
As a result, we use $T'=800$ in our main experiments, as it balances the private federated training and public training to have satisfactory distribution matching capabilities without biasing too much towards the public data distribution.

\begin{table}[htp!]\small
    % \resizebox{1.0\linewidth}{!}
    \centering
    {
        \begin{tabular}{l|lllll}
            \toprule
      $T'$                        &   $0$             &  $400$                   & $800$                &     $1200$    \\ \midrule
      $\varepsilon$=1.77  & $25.41$ & $26.43$   &  $\textbf{26.73}$ & $25.20$            \\
      $\varepsilon$=18.71 & $28.38$ & $29.55$   &  $\textbf{29.70}$ & $28.93$      \\
\bottomrule
\end{tabular}
}
\caption{\small Ablation studies on the timing ($T'$) of mid-point public training for on-device LSTM w/o distribution matching.}
\label{tab:t_prime_random_sampling}
\end{table}

\section{Detailed Theoretical Results}
\label{app:theory}

\subsection{Discussion on the distance metrics of log-density functions}

We need to define a meaningful distance metric in order to define the closeness of two log-density functions. 
To do this, we can choose any inner product $\langle \cdot, \cdot \rangle$ in the function space of $\gH=\{f:\gX\to \sR\}$. Note that the log-density functions  $\ell_{\text{pub}}, \ell_{\text{priv}}, \hat\ell_{\text{priv}}\in \gH$. 
Accordingly, the norm in the function space $\gH$ is denoted as $\|\cdot\|$ and by definition $\forall f\in \gH:\|f\|=\sqrt{\langle f,f\rangle}$.

We note that our analysis works for \textbf{any} choice of the inner product as long as they don't make the log-densities norm infinite. For a concrete example, we discuss a generalization of the $L^2$ inner product, i.e., the $L^\pi$ inner product where $\pi$ is a distribution on $\gX$.

Formally, for this example of $\gH=L^\pi$ we define $\langle f, g \rangle_\pi = \E_{x\sim \pi}[f(x)g(x)]$ and $\|f\|_\pi=\sqrt{\E_{x\sim \pi}[f(x)^2]}$.

The $L^\pi$ is a rather general definition that is common in the literature of Bayesian coresets~\cite{zhang2021bayesian, campbell2019automated} and kernel machine~\cite{rahimi2007random}. For example, it recovers $L^2$ if $\pi$ is chosen to be the uniform distribution on $\gX$. 

Moreover, if we choose $\pi=p_{\text{priv}}$ as the private data density, we can show that for any probability density function $p$, the  distance between $\log p$ and $\log p_{\text{priv}}$ measured by $L^{p_{\text{priv}}}$ norm upper bounds the KL divergence between $p_{\text{priv}}$ and $p$: 
\begin{align}
    \|\log p-\log p_{\text{priv}}\|^2_\pi&=\E_{x\sim p_{\text{priv}}}[(\log p(x)-\log p_{\text{priv}}(x))^2]=\E_{x\sim p_{\text{priv}}}\left(\log \frac{p(x)}{p_{\text{priv}}(x)}\right)^2\\
    &\geq \left(\E_{x\sim p_{\text{priv}}}\log \frac{p(x)}{p_{\text{priv}}(x)}\right)^2 \tag{Jensen's Inequality}\\
    &=(\texttt{KL}(p_{\text{priv}}| p))^2
\end{align}

In general, the distribution $\pi$ characterize where in $\gX$ we want to evaluate a function.   

Above we discuss a concrete choice of the inner product and the accordingly the norm to measure the distance between log-density functions. Since our analysis will work with any choice of inner product, we return to using the notation of $\langle \cdot, \cdot \rangle$  and $\|\cdot\|$ to remain generality in our main result. 

\subsection{Proof}
\begin{theorem}[Theorem~\ref{thm:main} Restated]
Let $\epsilon(\hat f)=\E[\|\hat f-\ell_{\text{priv}}\|^2]$ characterise how good $\hat f$ is as an estimator of the true private data log-density $\ell_{\text{priv}}$ for any random function $\hat f\in \gH$.
Consider the following three quantities: 
\begin{enumerate}[leftmargin=1.3em,topsep=1pt,noitemsep]
    \item $\epsilon(\ell_{\text{pub}})$ that characterizes the error if we use the public log-density function $\ell_{\text{pub}}$ to approximate the $\ell_{\text{priv}}$
    \item $\epsilon(\hat\ell_{\text{priv}})$ that characterizes the error if we use the noisy private log-density function $\hat\ell_{\text{priv}}$ to approximate the $\ell_{\text{priv}}$
    \item $\epsilon(\hat h)$ that characterizes the error if we use $\hat h=\frac{1}{2}\ell_{\text{pub}}+\frac{1}{2}\hat\ell_{\text{priv}}$ to approximate the $\ell_{\text{priv}}$.
\end{enumerate} 
Then,
\begin{align}
    \epsilon(\ell_{\text{pub}})&=d^2_{\text{pub, priv}}\\
    \epsilon(\hat\ell_{\text{priv}})&=\sigma_{\text{priv}}^2\\
    \epsilon(\hat h)&=\frac{1}{4}d^2_{\text{pub, priv}}+\frac{1}{4}\sigma_{\text{priv}}^2
\end{align}
\end{theorem}
\begin{proof}
We prove a general result which gives the theorem as special cases. For $\beta\in [0,1]$, define
\begin{align}
    \hat f_\beta=\beta \ell_{\text{pub}}+(1-\beta)\hat\ell_{\text{priv}}.
\end{align}

According to the definition of $\epsilon(\hat f_\beta)=\E[\|\hat f_\beta-\ell_{\text{priv}}\|^2]$, we have
    \begin{align}
        \epsilon(\hat f_\beta)&=\E[\|\hat f_\beta-\ell_{\text{priv}}\|^2]=\E[\|\beta\ell_{\text{pub}}+(1-\beta)\hat\ell_{\text{priv}}-\ell_{\text{priv}}\|^2]\\
        &=\E[\|\beta(\ell_{\text{pub}}-\ell_{\text{priv}})+(1-\beta)(\hat\ell_{\text{priv}}-\ell_{\text{priv}})\|^2]\\
        &=\beta^2 \|\ell_{\text{pub}}-\ell_{\text{priv}}\|^2 +(1-\beta)^2 \E \left[\|\hat\ell_{\text{priv}}-\ell_{\text{priv}}\|^2 \right]+2\beta(1-\beta)\E \left[\langle\ell_{\text{pub}}-\ell_{\text{priv}}, \hat\ell_{\text{priv}}-\ell_{\text{priv}} \rangle\right]\\
        &=\beta^2 d^2_{\text{pub, priv}}+(1-\beta)^2\sigma^2_{\text{priv}}+2\beta(1-\beta)\langle\ell_{\text{pub}}-\ell_{\text{priv}}, \E [\hat\ell_{\text{priv}}]-\ell_{\text{priv}} \rangle\\
        &=\beta^2 d^2_{\text{pub, priv}}+(1-\beta)^2\sigma^2_{\text{priv}}+0\\
        &=\beta^2 d^2_{\text{pub, priv}}+(1-\beta)^2\sigma^2_{\text{priv}}
    \end{align}
    Therefore, we can see that the theorem stands as we substitute $\hat f_1=\ell_{\text{pub}}$, $f_{\frac{1}{2}}=\hat h$, and $\hat f_0=\hat\ell_{\text{priv}}$.
\end{proof}

\subsection{Extended Analysis}
\label{app:extension}
Note that in the previous subsection the $\hat f_\beta$ is a weighted combination of $\ell_{\text{pub}}$ and $\hat\ell_{\text{priv}}$, \textit{i.e.}, $\hat f_\beta = (1-\beta)\ell_{\text{pub}}+\beta \hat\ell_{\text{priv}}$ where $\beta\in [0,1]$. Therefore,
one can show that with the optimal weight $\beta^\star$, it is guaranteed that $\epsilon(\hat f_{\beta^\star})\leq \min\{\epsilon(\ell_{\text{pub}}), \epsilon(\hat\ell_{\text{priv}})\}$.

This framework of analysis is general (as it stands with any meaningful inner product and its norm), and it may inspire even better ways to design estimators mitigating the domain shift and private model noise.

\end{document}